\documentclass[letterpaper]{article}

\usepackage{include/aaai18}
\usepackage{times}
\usepackage{helvet}
\usepackage{courier}
\usepackage{url}
\usepackage{graphicx}
\usepackage{subcaption} 
\usepackage{dblfloatfix}

\usepackage{booktabs}

\newcommand{\citet}[1]{\citeauthor{#1}~\shortcite{#1}}
\newcommand{\citep}{\cite}

\usepackage{comment}
\usepackage{amsmath}
\usepackage{amssymb}
\usepackage{amsthm}
\usepackage{siunitx}
\usepackage{wrapfig}
\usepackage{todonotes}
\usepackage{enumitem}
\usepackage{siunitx}
\usepackage{dsfont}
\usepackage{xcolor}
\usepackage{tablefootnote}
\usepackage[para,online,flushleft]{threeparttable}
\usepackage{multirow}
\usepackage{mathtools}

\newcommand{\myvec}[1]{\mathbf{#1}}
\newcommand{\myvecsym}[1]{\boldsymbol{#1}}

\newcommand{\vpi}{\myvecsym{\pi}}

\newcommand{\vtau}{\myvecsym{\tau}}

\newcommand{\vu}{\myvec{u}}

\newcommand{\vz}{\myvec{z}}

\newcommand{\vT}{\myvec{T}}
\newcommand{\vU}{\myvec{U}}

\newcommand{\E}{\mathbb{E}}

\newcommand{\be}{\begin{equation}}
\newcommand{\ee}{\end{equation}}
\newcommand{\bea}{\begin{eqnarray}}
\newcommand{\eea}{\end{eqnarray}}
\newcommand{\beaa}{\begin{eqnarray*}}
	\newcommand{\eeaa}{\end{eqnarray*}}

\DeclareMathAlphabet{\mathpzc}{OT1}{pzc}{m}{n}

\newcommand{\exs}[2]{{\mathbb E_{#1}}\left[ #2 \right]}

\usepackage{algorithm}
\usepackage{algorithmicx}
\usepackage[noend]{algpseudocode}
\usepackage{changepage}
\usepackage{bbm}

\newtheorem{lemma}{Lemma}

\newcommand{\mysoftmax}[1]{{\ensuremath{\text{softmax}(#1)}}}

\begin{document} 
	
	\title{Counterfactual Multi-Agent Policy Gradients}
	\author{
		Jakob N. Foerster$^{1,\dagger}$\\
		jakob.foerster@cs.ox.ac.uk\\
		\And
		Gregory Farquhar$^{1,\dagger}$ \\
		gregory.farquhar@cs.ox.ac.uk \\
		\AND
		Triantafyllos Afouras$^{1}$\\
		afourast@robots.ox.ac.uk\\
		\And
		Nantas Nardelli$^{1}$\\
		nantas@robots.ox.ac.uk\\
		\And
		Shimon Whiteson$^1$\\
		shimon.whiteson@cs.ox.ac.uk\\
		\AND
		\textnormal{$^1$University of Oxford, United Kingdom \quad 
		$^\dagger$Equal contribution}
	}
	
	\maketitle

\begin{abstract}
\label{sec:abstract}
Many real-world problems, such as network packet routing  and the coordination of autonomous vehicles, are naturally modelled as cooperative multi-agent systems.  There is a great need for new reinforcement learning methods that can efficiently learn
decentralised policies for such systems.  To this end, we propose a new multi-agent actor-critic method called \emph{counterfactual multi-agent} (COMA) policy
gradients.  COMA uses a centralised critic to estimate the $Q$-function and
decentralised actors to optimise the agents' policies.  In addition, to address
the challenges of multi-agent credit assignment, it uses a \emph{counterfactual
baseline} that marginalises out a single agent's action, while keeping the other
agents' actions fixed. COMA also uses a critic representation that allows the
counterfactual baseline to be computed efficiently in a single forward pass. We
evaluate COMA in the testbed of \emph{StarCraft unit micromanagement}, using a
decentralised variant with significant partial observability. COMA significantly
improves average performance over other multi-agent actor-critic methods in this
setting,
and the best performing agents are competitive with  state-of-the-art centralised
controllers that get access to the full state.\end{abstract}

\section{Introduction}
\label{sec:intro}

Many complex \emph{reinforcement learning} (RL) problems such as the coordination of autonomous vehicles \citep{cao2013overview}, network packet delivery \citep{ye2015multi}, and distributed logistics \citep{ying2005multi} are naturally modelled as cooperative multi-agent systems.  However, RL methods designed for single agents typically fare poorly on such tasks, since the joint action space of the agents grows exponentially with the number of agents.  

To cope with such complexity, it is often necessary to resort to \emph{decentralised policies}, in which each agent selects its own action conditioned only on its local action-observation history.  Furthermore, partial observability and communication constraints during execution may necessitate the use of decentralised policies even when the joint action space is not prohibitively large.

Hence, there is a great need for new RL methods that can efficiently learn decentralised policies.  In some settings, the learning itself may also need to be decentralised.  However, in many cases, learning can take place in a simulator or a laboratory in which extra state information is available and agents can communicate freely.  This \emph{centralised training of decentralised policies} is a standard paradigm for multi-agent planning \citep{Oliehoek08JAIR,kraemer2016multi} and has recently been picked up by the deep RL community \citep{foerster2016learning,jorge2016learning}.  However, the question of how best to exploit the opportunity for centralised learning remains open.

Another crucial challenge is \emph{multi-agent credit assignment} 
\citep{chang2003all}: in cooperative settings, joint actions typically generate 
only global rewards, making it difficult for each agent to deduce its own 
contribution to the team's success.  Sometimes it is possible to design 
individual reward functions for each agent.  However, these rewards are not 
generally available in cooperative settings and often fail to encourage 
individual agents to sacrifice for the greater good. This often substantially 
impedes multi-agent learning in challenging tasks, even with relatively small 
numbers of agents.

In this paper, we propose a new multi-agent RL method called 
\emph{counterfactual multi-agent} (COMA) policy gradients, in order to address 
these issues.  COMA takes an \emph{actor-critic} \citep{konda2000actor} 
approach, in which the \emph{actor}, i.e., the policy, is trained by following 
a gradient estimated by a \emph{critic}.
COMA is based on three main ideas.  

First, COMA uses a centralised critic. The critic is only used during learning, 
while only the actor is needed during execution. Since learning is centralised, 
we can therefore use a centralised critic that conditions on the joint action 
and all available state information, while each agent's policy conditions only 
on its own action-observation history.

Second, COMA uses a \emph{counterfactual baseline}.  The idea is inspired by 
\emph{difference rewards} \citep{wolpert2002optimal,tumer2007distributed}, in 
which each agent 
learns from a shaped reward that compares the global reward to the reward 
received when that agent's action is replaced with a \emph{default action}.  
While difference rewards are a powerful way to perform multi-agent credit 
assignment, they require access to a simulator or estimated reward 
function, and in general it is unclear how to choose the default 
action. COMA addresses this by using the centralised critic to compute an agent-specific
\emph{advantage function} that compares the estimated return for the 
current joint action  to a counterfactual baseline that marginalises out a 
single 
agent's action, while keeping the other agents' actions fixed. This is similar 
to calculating an \emph{aristocrat utility} \citep{wolpert2002optimal}, but 
avoids the problem of a recursive interdependence between the policy and utility function 
because the expected contribution of the counterfactual baseline to 
the policy gradient is zero.
Hence, instead of relying on extra simulations, approximations, or assumptions 
regarding appropriate default actions, COMA computes a separate 
baseline for each agent that relies on the centralised critic to reason about
counterfactuals in which only that agent's action changes.

Third, COMA uses a critic representation that allows the counterfactual baseline to be computed efficiently.  In a single forward pass, it computes the $Q$-values for all the different actions of a given agent, conditioned on the actions of all the other agents.   Because a single centralised critic is used for all agents, all $Q$-values for all agents can be computed in a single batched forward pass.

We evaluate COMA in the testbed of \emph{StarCraft unit 
micromanagement}\footnote{StarCraft and its expansion StarCraft: Brood War are 
trademarks of Blizzard Entertainment\texttrademark.}, which has recently 
emerged as a challenging RL benchmark task with high stochasticity, a large 
state-action space, and delayed rewards. Previous works  
\citep{usunier2016episodic,peng2017multiagent} have made use of a centralised 
control policy that conditions on the entire state and can use powerful 
macro-actions, using StarCraft's built-in planner, that combine movement and 
attack actions. To produce a meaningfully decentralised benchmark that proves 
challenging for scenarios with even relatively few agents, we propose a variant 
that massively reduces each agent's field-of-view and removes access to these 
macro-actions.

Our empirical results on this new benchmark show that COMA can significantly 
improve performance over other multi-agent actor-critic methods, as well as 
ablated versions of COMA itself. In addition, COMA's  best agents are 
competitive with state-of-the-art centralised controllers that are given 
access to full state information and macro-actions.

\section{Related Work}
\label{sec:related}

Although multi-agent RL has been applied in a variety of settings~\citep{busoniu2008comprehensive,yang2004multiagent}, it has often been restricted to tabular methods and simple environments.
One exception is recent work in deep multi-agent RL, which can scale to high 
dimensional input and action spaces. \citet{tampuu2015multiagent} use a 
combination of DQN with independent 
$Q$-learning~\citep{tan1993multi,MASfoundations09} to learn how to play 
two-player pong. More recently the same method has been used 
by~\citet{leibo2017multi} to study the emergence of collaboration and defection 
in sequential social dilemmas.

Also related is work on the emergence of communication between agents, learned 
by gradient descent 
\citep{das2017learning,mordatch2017emergence,lazaridou2016multi,foerster2016learning,sukhbaatar2016learning}.
 In this line of work, passing gradients between agents during training and 
sharing parameters are two common ways to take advantage of centralised 
training. However, these methods do not allow for extra state information to be 
used during learning and do not address the multi-agent credit assignment 
problem.

\citet{guptacooperative} investigate actor-critic methods for decentralised execution with centralised training. However, in their methods both the actors and the critic condition on local, per-agent, observations and actions, and multi-agent credit assignment is addressed only with hand-crafted local rewards.

Most previous applications of RL to StarCraft micromanagement use a centralised 
controller, with access to the full state, and control of all units, although 
the 
architecture of the controllers exploits the multi-agent nature of the problem.
\citet{usunier2016episodic} use a \emph{greedy MDP}, which at each timestep
sequentially chooses actions for agents given all previous actions, in
combination with zero-order optimisation, while \citet{peng2017multiagent} use
an actor-critic method that relies on RNNs to exchange information between the
agents.

The closest to our problem setting is that of
\citet{foerster2017stabilising}, who also use a multi-agent representation and
decentralised policies. However, they focus on stabilising experience replay
while using DQN and do not make full use of the centralised training regime. As
they do not report on absolute win-rates we do not compare performance directly.
However, \citet{usunier2016episodic} address similar scenarios to our
experiments and implement a DQN baseline in a fully observable setting. In Section \ref{sec:results} we
therefore report our competitive performance against these state-of-the-art
baselines, while maintaining decentralised
control. \citet{omidshafiei2017deep} also address the stability of experience
replay in multi-agent settings, but assume a fully decentralised training
regime.

\cite{lowe2017multi} concurrently propose a multi-agent policy-gradient 
algorithm using centralised critics. Their approach does not address  
multi-agent credit assignment. Unlike our work, it learns a separate 
centralised critic for each agent and is applied to competitive 
environments with continuous action spaces. 

Our work builds directly off of the idea of \emph{difference rewards} 
\citep{wolpert2002optimal}.  The relationship of COMA to 
this line of work is discussed in Section \ref{sec:methods}.

\section{Background}
\label{sec:background}

We consider a fully cooperative multi-agent task that can be described as a 
stochastic game $G$, defined by a tuple $G = {\langle}S, U, P, r, Z, O, n, 
\gamma{\rangle}$, in which $n$ agents identified by $a \in A \equiv 
\{1,...,n\}$ choose sequential actions. The environment has a true state $s \in 
S$. At each time step, each agent simultaneously chooses an action $u^a \in U$, 
forming a joint action $\mathbf{u} \in \mathbf{U} \equiv U^{n}$ which induces a 
transition in the environment according to the state transition function 
$P(s'|s,\mathbf{u}): S \times \mathbf{U} \times S \rightarrow [0,1]$. The 
agents all share the same reward function $r(s,\mathbf{u}): S \times \mathbf{U} 
\rightarrow \mathbb{R}$ and $\gamma \in [0,1)$ is a discount factor.

We consider a partially observable setting, in which agents draw observations $z \in Z$ according to the observation function $O(s, a): S \times A \rightarrow Z$. Each agent has an action-observation history $\tau^a \in T \equiv (Z \times U)^{*}$, on which it conditions a stochastic policy $\pi^a(u^a|\tau^a): T \times U \rightarrow [0,1]$. We denote joint quantities over agents in bold, and joint quantities over agents other than a given agent $a$ with the superscript ${-a}$.

The discounted return is $R_t = \sum_{l=0}^\infty \gamma^l r_{t+l}$. The agents' joint policy induces a value function, i.e., an expectation over $R_t$, $V^{\vpi}(s_t) = \exs{s_{t+1:\infty},\vu_{t:\infty}}{R_t|s_t}$, and an action-value function $Q^{\vpi}(s_t, \vu_t) = \exs{s_{t+1:\infty},\vu_{t+1:\infty}}{R_t|s_t,\vu_t}$. The advantage function is given by $A^{\vpi}(s_t, \vu_t) = Q^{\vpi}(s_t, \vu_t) - V^{\vpi}(s_t)$.

Following previous work 
\citep{Oliehoek08JAIR,kraemer2016multi,foerster2016learning,jorge2016learning}, 
our problem setting allows centralised training but requires 
decentralised execution. This is a natural paradigm for a large set of 
multi-agent problems where training is carried out using a simulator with 
additional state information, but the agents must rely on local 
action-observation histories during execution. To condition on this full 
history, a deep RL agent may make use of a recurrent neural network 
\citep{hausknecht2015deep}, typically with a gated model such as LSTM 
\citep{hochreiter1997long} or GRU \citep{cho2014properties}.

In Section \ref{sec:methods}, we develop a new multi-agent policy gradient 
method for tackling this setting.  In the remainder of this section, we provide 
some background on single-agent policy gradient methods 
\citep{sutton1999policy}. Such methods optimise a single agent's policy, 
parameterised by $\theta^\pi$, by performing gradient ascent on an estimate of 
the expected discounted total reward $J = \exs{\pi}{R_0}$. Perhaps the simplest 
form of policy gradient is REINFORCE \citep{williams1992simple}, in which the 
gradient is:
\begin{equation}
g = \exs{s_{0:\infty},u_{0:\infty}}{\sum_{t=0}^{T} R_t \nabla_{\theta^{\pi}} \log \pi (u_t \vert s_t) }.
\end{equation}
In \emph{actor-critic} approaches 
\citep{sutton1999policy,konda2000actor,DBLP:journals/corr/SchulmanMLJA15},
 the \emph{actor}, i.e., the policy, is trained by following a gradient that 
depends on a \emph{critic}, which usually estimates a value function. In 
particular, $R_t$ is replaced by any expression equivalent to $Q(s_t, u_t) - 
b(s_t)$, where $b(s_t)$  is a baseline designed to reduce variance 
\citep{weaver2001optimal}.  A common choice is $b(s_t) = V(s_t)$, in which case 
$R_t$ is replaced by $A(s_t, u_t)$.  Another option is to replace $R_t$ with 
the \emph{temporal difference} (TD) error $r_t + \gamma V(s_{t+1}) - V(s)$, 
which is an unbiased estimate of $A(s_t, u_t)$. In practice, the gradient must 
be estimated from trajectories sampled from the environment, and the 
(action-)value functions must be estimated with function approximators. 
Consequently, the bias and variance of the gradient estimate depends strongly 
on the exact choice of estimator \citep{konda2000actor}. 

In this paper, we  train critics $f^c(\cdot,\theta^c)$ on-policy to estimate 
either $Q$ or $V$, using a variant of TD($\lambda$) \citep{sutton1988learning} 
adapted for use with deep neural networks. TD($\lambda$) uses a mixture of 
$n$-step returns $G_t^{(n)} = \sum_{l=1}^{n} \gamma^{l-1}r_{t+l} + \gamma^n 
f^c(\cdot_{t+n}, \theta^c)$. In particular, the critic parameters $\theta^c$ 
are updated by minibatch gradient descent to minimise the following loss:
\begin{equation}
\mathcal{L}_t(\theta^c) = ( y^{(\lambda)} - f^c(\cdot_t, \theta^c))^{2},
\end{equation}
where $y^{(\lambda)} = (1-\lambda)\sum_{n=1}^{\infty} \lambda^{n-1} G_t^{(n)}$, 
and the $n$-step returns $G_t^{(n)}$ are calculated with bootstrapped values 
estimated by a \emph{target network} \citep{mnih2015human} with parameters 
copied periodically from $\theta^c$.

\section{Methods}
\label{sec:methods}

In this section, we describe approaches for extending policy gradients to our multi-agent setting.

\subsection{Independent Actor-Critic}
The simplest way to apply policy gradients to multiple agents is to have each agent learn independently, with its own actor and critic, from its own action-observation history.  This is essentially the idea behind \emph{independent Q-learning} \citep{tan1993multi}, which is perhaps the most popular multi-agent learning algorithm, but with actor-critic in place of $Q$-learning.  Hence, we call this approach \emph{independent actor-critic} (IAC).  

In our implementation of IAC, we speed learning by sharing parameters among the 
agents, i.e., we learn only one actor and one critic, which are used by all 
agents. The agents can still behave differently because they receive different 
observations, including an agent-specific ID, and thus evolve different hidden 
states.  Learning remains independent in the sense that each agent's critic 
estimates only a local value function, i.e., one that conditions on $u^a$, not 
$\mathbf{u}$. Though we are not aware of previous applications of this specific 
algorithm, we do not consider it a significant contribution but instead merely 
a baseline algorithm.

We consider two variants of IAC. In the first, each agent's critic estimates $V(\tau^a)$ and follows a gradient based on the TD error, as described in Section \ref{sec:background}.  In the second, each agent's critic estimates $Q(\tau^a,u^a)$ and follows a gradient based on the advantage: $A(\tau^a, u^a) = Q(\tau^a,u^a)  - V(\tau^a)$, where $V(\tau^a) = \sum_{u^a}  \pi(u^a | \tau^a) Q(\tau^a,u^a)$. Independent learning is straightforward, but the lack of information sharing at training time makes it difficult to learn coordinated strategies that depend on interactions between multiple agents, or for an individual agent to estimate the contribution of its actions to the team's reward. 

\begin{figure*}[ht]
\centering
	\includegraphics[width=0.9\linewidth]{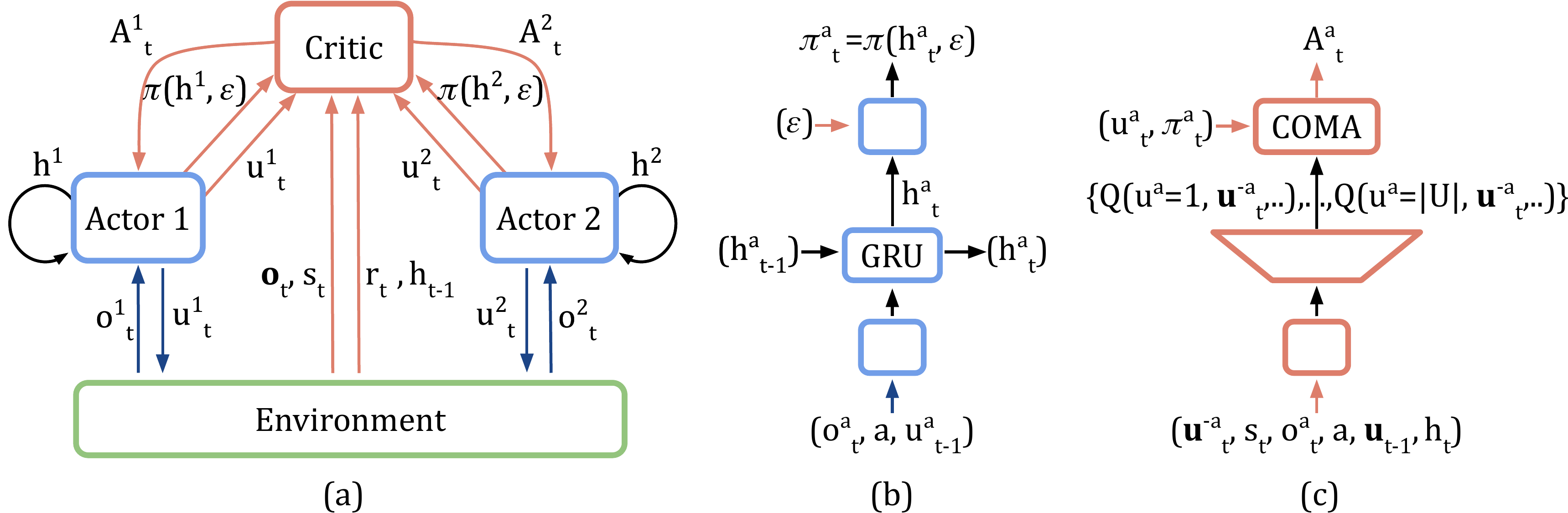}
\caption{In (a), information flow between the decentralised actors, the environment and the centralised critic in COMA; red arrows and components are only required during centralised learning. In (b) and (c), architectures of the actor and critic. 
}
\label{fig:fig_1}
\end{figure*}

\subsection{Counterfactual Multi-Agent Policy Gradients}
The difficulties discussed above arise because, beyond parameter sharing, IAC fails to exploit the fact that learning is centralised in our setting.  In this section, we propose \emph{counterfactual multi-agent} (COMA) policy gradients, which overcome this limitation.  Three main ideas underly COMA: 1) centralisation of the critic, 2) use of a counterfactual baseline, and 3) use of a critic representation that allows efficient evaluation of the baseline.  The remainder of this section describes these ideas.

First, COMA uses a centralised critic.  Note that in IAC, each actor $\pi(u^a | \tau^a)$  and each critic $ Q(\tau^a,u^a) $ or $ V(\tau^a)$ conditions only on the agent's own action-observation history $\tau^a$.  However, the critic is used only during learning and only the actor is needed during execution.  Since learning is centralised, we can therefore use a centralised critic that conditions on the true global state $s$, if it is available, as well as the joint action-observation histories $\vtau$. Each actor conditions on its own action-observation histories $\tau^a$, with parameter sharing, as in IAC.  Figure \ref{fig:fig_1}a illustrates this setup.

A naive way to use this centralised critic would be for each actor to follow a gradient based on the TD error estimated from this critic:

\begin{equation}
g = \nabla_{\theta^{\pi}}  \log  \pi({u^a_t|\tau^a_t} )\left( r + \gamma V(s_{t+1}) - V(s_t) \right).
\end{equation}

 However, such an approach fails to address a key credit assignment problem.  
 Because the TD error considers only global rewards, the gradient computed for 
 each actor does not explicitly reason about how that particular agent's 
 actions contribute to that global reward.  Since the other agents may be 
 exploring, the gradient for that agent becomes very noisy, particularly when 
 there are many agents.

Therefore, COMA uses a \emph{counterfactual baseline}.  The idea is inspired by 
\emph{difference rewards} \citep{wolpert2002optimal}, in which each agent 
learns from a shaped reward $D^a = r(s,\vu) - r(s, (\vu^{-a},c^a))$ that 
compares the global reward to the reward received when the action of agent $a$ 
is replaced with a \emph{default action} $c^a$. Any action by agent $a$ that 
improves $D^a$ also improves the true global reward $r(s, \vu)$, because $r(s, 
(\vu^{-a},c^a))$ does not depend on agent $a$'s actions.

Difference rewards are a powerful way to perform multi-agent credit 
assignment.  However, they typically require access to a simulator in order to 
estimate $r(s, (\vu^{-a},c^a))$.  When a simulator is already being used for 
learning, difference rewards increase the number of simulations that must be 
conducted, since each agent's difference reward requires a separate 
counterfactual simulation. \citet{proper2012modeling} and 
\citet{colby2015approximating} 
propose estimating difference rewards using function approximation rather than a 
simulator. However, this still requires a 
user-specified default action $c^a$ that can be difficult to choose in many 
applications. In an actor-critic architecture, this approach would also
introduce an additional source of approximation error.

A key insight underlying COMA is that a centralised critic can be used to 
implement difference rewards in a way that avoids these problems.  COMA learns 
a centralised critic, $Q(s, \vtau, \vu)$ that estimates $Q$-values for the joint 
action $\vu$ conditioned on the central state $s$ and the joint action-observation history. For each agent $a$ we can 
then compute an advantage function that compares the $Q$-value for the current 
action $u^a$ to a counterfactual baseline that marginalises out $u^a$, while 
keeping the other agents' actions $\vu^{-a}$ fixed:
\begin{equation}
	\label{eqn:advantage}
	A^a(s, \vtau, \vu) = Q(s, \vtau, \vu ) - \sum_{u'^a} \pi^a(u'^a \vert \tau^a) 
	Q(s, \vtau, (\vu^{-a},u'^a)).
\end{equation}
Hence, $A^a(s, \vtau, u^a)$ computes a separate baseline for each agent that uses the 
centralised critic to reason about counterfactuals in which only $a$'s action 
changes, learned directly from agents' experiences instead of relying on extra 
simulations, a reward model, or a user-designed default action.

This advantage has the same form as the \emph{aristocrat utility} 
\citep{wolpert2002optimal}. However, optimising for an aristocrat utility using 
value-based methods creates a self-consistency problem because the policy and utility 
function depend recursively on each other. As a result, prior work focused
on difference evaluations using default states and actions. COMA is different because the
 counterfactual baseline's expected contribution to the gradient, as with other policy gradient baselines, is zero.  Thus, while
the baseline does depend on the policy, its expectation does not. Consequently, COMA can use 
this form of the advantage without creating a self-consistency problem.

While COMA's advantage function replaces potential extra simulations with 
evaluations of the critic, those evaluations may themselves be expensive if the 
critic is a deep neural network.  Furthermore, in a typical representation, the 
number of output nodes of such a network would equal $\vert U \vert ^ n$, the 
size of the joint action space, making it impractical to train.  To address 
both these issues, COMA uses a critic representation that allows for efficient 
evaluation of the baseline. In particular, the actions of the other agents, 
$\vu^{-a}_t$, are part of the input to the network, which outputs a $Q$-value 
for each of agent $a$'s actions, as shown in Figure~\ref{fig:fig_1}c. 
Consequently, the counterfactual advantage can be calculated efficiently by a 
single forward pass of the actor and critic, for each agent. Furthermore, the 
number of outputs is only $\vert U \vert$ instead of ($\vert U \vert ^ n$). 
While the network has a large input space that scales linearly in the number of 
agents and actions, deep neural networks can generalise well across such spaces.

In this paper, we focus on settings with discrete actions. However, COMA can be 
easily extended to continuous actions spaces by estimating the expectation in  
\eqref{eqn:advantage} with Monte Carlo samples or using functional forms that 
render it analytical, e.g., Gaussian policies and critic.

The following lemma establishes the convergence of COMA to a locally optimal policy.  The proof follows directly from the convergence of single-agent actor-critic algorithms \citep{sutton1999policy}, and is subject to the same assumptions. \citeauthor{lyu2024centralizedcriticsmultiagentreinforcement} (\citeyear{lyu2024centralizedcriticsmultiagentreinforcement}, Appendix E.2) prove a related result, showing that a family of cooperative policy gradient methods with centralized critics, which includes COMA, converge to local optima assuming access to the correct policy values, i.e., a perfect critic.

\begin{lemma}
\label{lemma}
For an actor-critic algorithm with a compatible TD(1) critic following a COMA 
policy gradient
\begin{equation}
g_k = \E_{\vpi} \left[ \sum_a \nabla_{\theta_k} \log
\pi^a(u^a | \tau^a)
A^a(s,\vtau, \vu) \right]
\end{equation}
at each iteration $k$,
\begin{equation}
\lim \inf_k ||\nabla J|| = 0 \quad w.p. \; 1.
\end{equation}
\end{lemma}
\begin{proof}
See Appendix \ref{app:proof}.
\end{proof}

\section{Experimental Setup}
\label{sec:setting}

In this section, we describe the StarCraft problem to which we apply COMA, as
well as details of the state features, network architectures, training regimes,
and ablations.

\textbf{Decentralised StarCraft Micromanagement.} StarCraft is a rich
environment with stochastic dynamics that cannot be easily emulated. Many
simpler multi-agent settings, such as Predator-Prey \citep{tan1993multi} or
Packet World \citep{weyns2005packet}, by contrast, have full simulators with
controlled randomness that can be freely set to any state in order to perfectly
replay experiences. This makes it possible, though computationally expensive, to
compute difference rewards via extra simulations. In StarCraft, as in the real
world, this is not possible.

In this paper, we focus on the problem of \emph{micromanagement} in StarCraft,
which refers to the low-level control of individual units' positioning and
attack commands as they fight enemies. This task is naturally represented as a
multi-agent system, where each StarCraft unit is replaced by a decentralised
controller. We consider several scenarios with symmetric teams formed of: 3
marines (3m), 5 marines (5m), 5 wraiths (5w), or 2 dragoons with 3 zealots 
(2d\_3z). The enemy team is controlled by the StarCraft AI, which uses 
reasonable but suboptimal hand-crafted heuristics.

We allow the agents to choose from a set of discrete actions:
\texttt{move[direction]}, \texttt{attack[enemy\_id]}, \texttt{stop}, and
\texttt{noop}. In the StarCraft game, when a unit selects an attack action, it
first moves into attack range before firing, using the game's built-in
pathfinding to choose a route. These powerful \emph{attack-move} macro-actions
make the control problem considerably easier.

To create a more challenging benchmark that is meaningfully decentralised, we
impose a restricted field of view on the agents, equal to the firing range of
ranged units' weapons, shown in Figure~\ref{fig:setup}. This departure from the standard setup for
centralised StarCraft control has three effects.

\begin{figure}[h!]
    \begin{center}
        \includegraphics[width=1\linewidth]{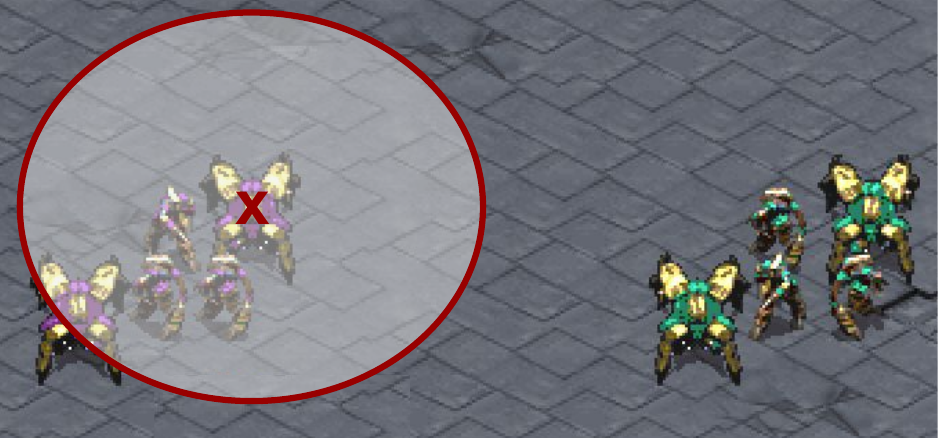}
    \end{center}
    \caption{Starting position with example local field of view for the 2d\_3z map.}
    \label{fig:setup}
\end{figure}

First, it introduces significant partial observability. Second, it means units
can only attack when they are in range of enemies, removing access to the
StarCraft macro-actions. Third, agents cannot distinguish between enemies who
are dead and those who are out of range and so can issue invalid attack commands
at such enemies, which results in no action being taken. This substantially
increases the average size of the  action space, which in turn increases the
difficulty of both exploration and control.

Under these difficult conditions, scenarios with even relatively small numbers
of units become much harder to solve. As seen in Table \ref{tbl:test_in_domain},
we compare against a simple hand-coded heuristic that instructs the agents to
run forwards into range and then focus their fire, attacking each enemy in turn
until it dies. This heuristic achieves a 98\% win rate on 5m with a full field
of view, but only 66\% in our setting. To perform well in this task, the agents
must learn to cooperate by positioning properly and focussing their fire,
while remembering which enemy and ally units are alive or out of
view.

All agents receive the same global reward at each time step, equal to the sum of
damage inflicted on the opponent units minus half the damage taken. Killing an
opponent generates a reward of 10 points, and winning the game generates a
reward equal to the team's remaining total health plus 200. This damage-based
reward signal is comparable to that used by \citet{usunier2016episodic}. Unlike 
\cite{peng2017multiagent}, our approach does not require estimating local 
rewards.

 \begin{figure*}[t!]
 	\centering
 	\begin{subfigure}[b]{0.4\linewidth}
 		\includegraphics[width=\textwidth]{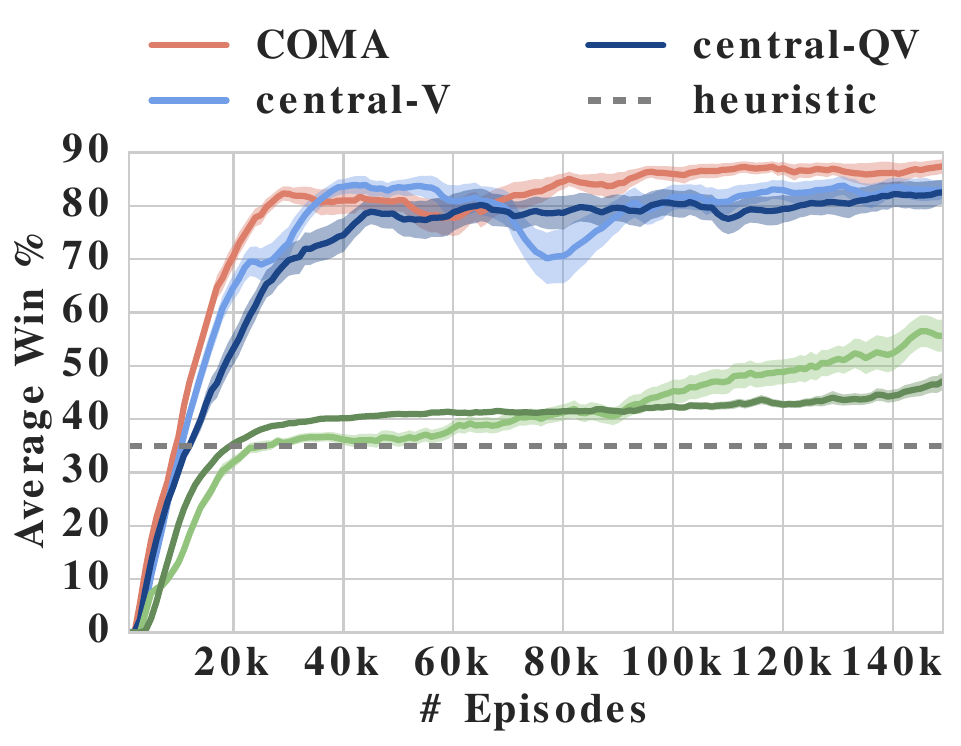}
 		\caption{3m}
 	\end{subfigure}
  	\begin{subfigure}[b]{0.4\linewidth}
	\includegraphics[width=\textwidth]{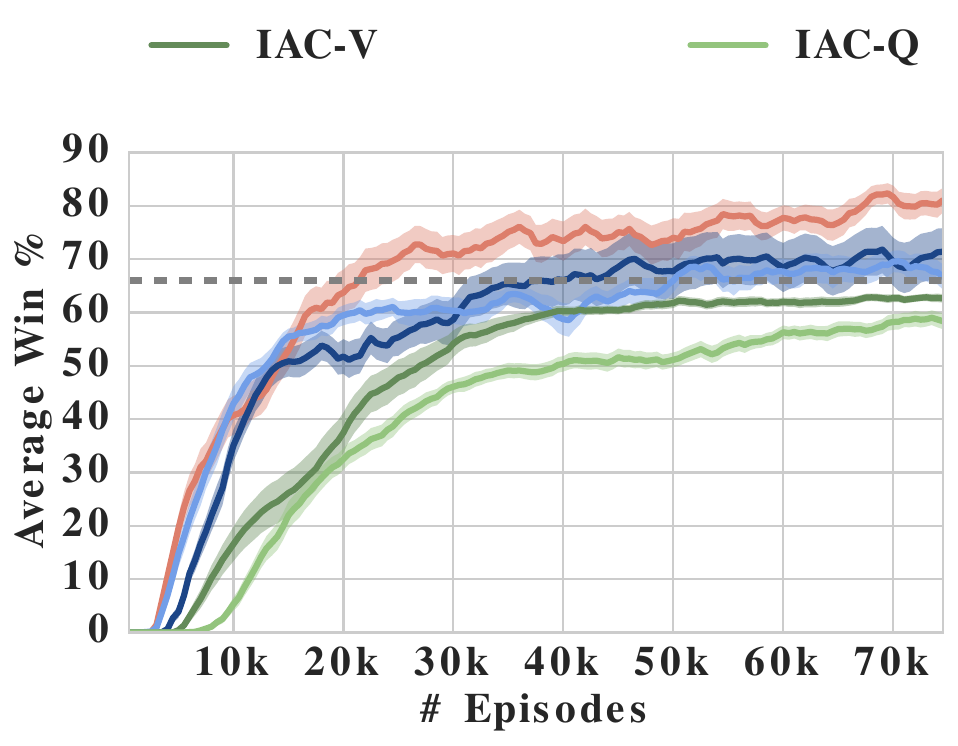}
 	\caption{5m}
 	\end{subfigure}
 	\begin{subfigure}[b]{0.4\linewidth}
	\includegraphics[width=\textwidth]{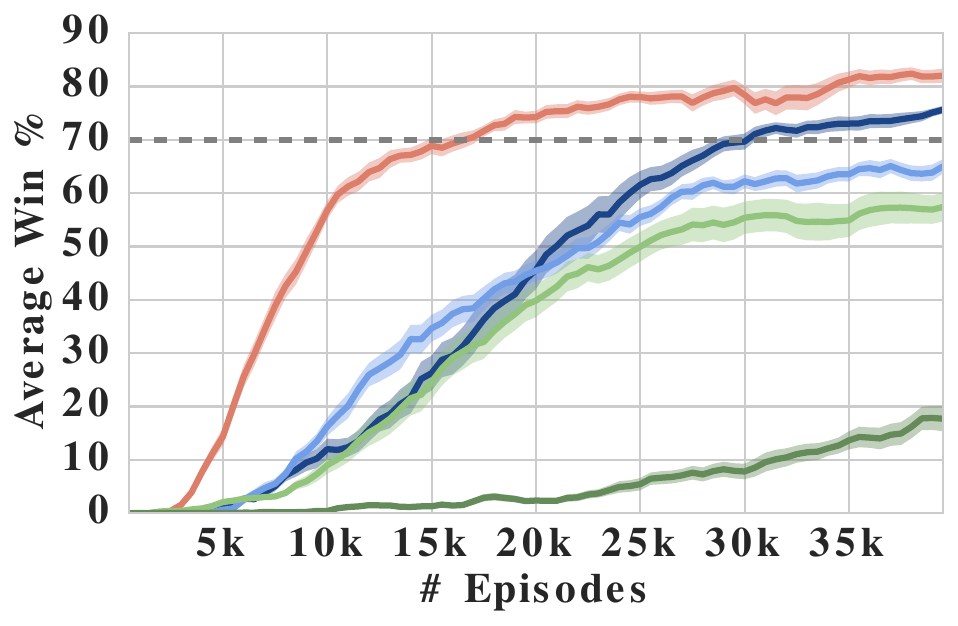}
	\caption{5w}
	\end{subfigure}
 	\begin{subfigure}[b]{0.4\linewidth}
	\includegraphics[width=\textwidth]{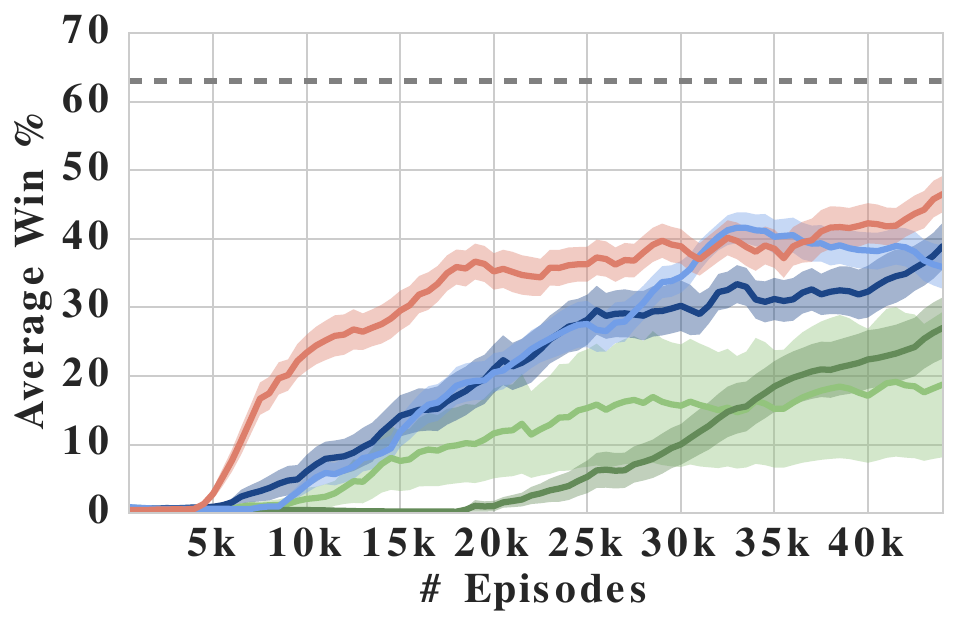}
	\caption{2d\_3z}

	\end{subfigure}
 	\caption{Win rates for COMA and competing algorithms on four different 
 		scenarios. COMA outperforms all baseline methods. Centralised critics 
 		also
 		clearly outperform their decentralised counterparts. The legend at the 
 		top applies across all plots.}
 	\label{fig:fig_2}
    \vspace{-0.75em}
 \end{figure*}

\textbf{State Features.}
\label{ssec:features}
The actor and critic receive different input features, corresponding to local
observations and global state, respectively. Both include features for allies
and enemies. \emph{Units} can be either allies or enemies, while \emph{agents} are
the decentralised controllers that command ally units.

The local observations for every agent are drawn only from a circular subset of
the map centred on the unit it controls and include for each unit within this 
field of view:
\texttt{distance}, \texttt{relative x}, \texttt{relative
y}, \texttt{unit type} and \texttt{shield}.\footnote{After firing, a unit's 
\texttt{cooldown} is reset, and it must drop
before firing again. Shields absorb damage until they break, after which units
start losing health. Dragoons and zealots have shields but marines do not.}
All features are normalised by their maximum values. We do not include any 
information about the units' current target.

The global state representation consists of similar features, but for
all units on the map regardless of fields of view. Absolute distance is not 
included, and $x$-$y$ locations are given relative to the centre of the map 
rather than to a particular agent. The global state also includes 
\texttt{health points} and \texttt{cooldown} for all agents. The representation 
fed
to the centralised $Q$-function critic is the concatenation of the global state 
representation with the 
local observation of the agent whose actions are being evaluated. Our 
centralised critic that estimates $V(s)$, 
and is therefore agent-agnostic, receives the global state concatenated with 
all agents' observations. The observations contain no new information but 
include the egocentric distances relative to that agent.

\textbf{Architecture \& Training.}
\label{ssec:architecture} 
The actor consists of 128-bit \emph{gated recurrent units} 
(GRUs)~\citep{cho2014properties} that use fully connected layers both to 
process the input and to produce the output values from the hidden state, 
$h^a_t$. The IAC critics use extra output heads appended to the last layer of 
the actor network. Action probabilities are produced from the final layer, 
$\vz$, via a bounded softmax distribution that lower-bounds the probability of 
any given action by $\epsilon / |U|$:  $P(u) = (1- \epsilon) \mysoftmax{\vz}_u 
+  \epsilon / |U|)$. We anneal $\epsilon$ linearly from $0.5$ to $0.02$ across 
$750$ training episodes. The centralised critic is a feedforward network with 
multiple ReLU layers combined with fully connected layers.  Hyperparameters 
were coarsely tuned on the 5m scenario and then used for all other 
maps. We found that the most sensitive parameter was TD($\lambda$), but settled 
on $\lambda=0.8$, which worked best for both COMA 
and our baselines. Our implementation uses 
TorchCraft~\citep{synnaeve2016torchcraft} and \mbox{Torch 7}~\citep{torch}. Pseudocode and
further details on the training procedure are in the 
supplementary material.

We experimented with critic architectures that are factored at the agent level and 
further exploit internal parameter sharing. However, we found that the 
bottleneck for scalability was not the centralisation of the critic, but rather 
the difficulty of multi-agent exploration.  Hence, we defer further investigation of factored COMA critics to future work.

\begin{table*}[ht]
	\begin{center}
		\resizebox{\textwidth}{!}{
			\begin{tabular}{l ccccccc ccc}
				\toprule
				& \multicolumn{7}{c}{Local Field of View (FoV)} & 
				\multicolumn{3}{c}{Full FoV, Central Control} \\
				\cmidrule(lr){2-8} \cmidrule(lr){9-11}
				
				\multirow{2}{*}{map}  & \multirow{2}{*}{heur.}    & 
				\multirow{2}{*}{IAC-$V$}  & \multirow{2}{*}{IAC-$Q$}  & 
				\multirow{2}{*}{cnt-$V$} & \multirow{2}{*}{cnt-$QV$}  & 
				\multicolumn{2}{c}{COMA}  & \multirow{2}{*}{heur.} & 
				\multirow{2}{*}{DQN} & \multirow{2}{*}{GMEZO}  \\
				&   &  &    &   &  & mean & best  & & & \\

				\midrule
				3m        & 35   & 47 (3)   & 56 (6)    & 83 (3)   & 83 (5)   & 
				\textbf{87} (3)  &  98  &  74  &   - & -          \\
				5m        & 66   & 63 (2)   & 58 (3)    & 67 (5)   & 71 (9)   & 
				\textbf{81} (5)   &  95  &  98  &  99 & 100 \\
				5w        & 70   & 18 (5)   & 57 (5)     & 65 (3)   & 76 (1)   
				& 
				\textbf{82} (3)   &  98  &  82  &    70 & 
				74\footnotemark[3] \\
				2d\_3z    & \textbf{ 63}   & 27 (9)   & 19 (21)  & 36 (6)  & 
				39 (5)   
				&47  (5)  &  65  &  68  &    61 & 90 \\
				\bottomrule
				
			\end{tabular}
		}
		



	\end{center}
	\caption{Mean win percentage averaged across final 1000 evaluation episodes 
		for the different maps, for all methods and the hand-coded 
		heuristic in 
		the decentralised setting with a limited field of view. The 
		highest mean performances are in bold, while values in 
		parentheses denote the 95\% confidence interval, for example 
		$87(3) = 87 \pm 3$.
		Also shown, maximum win percentages for COMA (decentralised), in 
		comparison to the heuristic and published results (evaluated in the 
		centralised setting).}
    \vspace{-0.75em}
	\label{tbl:test_in_domain}
	
\end{table*}

\textbf{Ablations.}
We perform ablation experiments to validate three key elements of COMA. First, 
we test the importance of centralising the critic by comparing 
against two IAC variants, IAC-$Q$ and IAC-$V$. These critics take the same 
decentralised input as the actor, and share parameters with the actor network 
up to the final layer. IAC-$Q$ then outputs $\vert U \vert$ $Q$-values, one for 
each action, while IAC-$V$ outputs a single state-value. Note that we still 
share parameters between agents, using the egocentric observations and ID's as 
part of the input to allow different behaviours to emerge. The cooperative 
reward function is still shared by all agents.

Second, we test the significance of learning $Q$ instead of $V$. The method 
\mbox{\emph{central-$V$}} still uses a central state for the critic, but learns 
$V(s)$, and uses the TD error to estimate the advantage for policy gradient 
updates.

Third, we test the utility of our counterfactual baseline. The method 
\mbox{\emph{central-$QV$}} learns both $Q$ and $V$ simultaneously and estimates 
the 
advantage as $Q-V$, replacing COMA's counterfactual baseline with $V$.
All methods use the same architecture and training scheme for the actors, and 
all critics are trained with TD($\lambda$).

\section{Results}
\label{sec:results}

Figure \ref{fig:fig_2} shows average win rates as a function of episode
for each method and each StarCraft scenario. For each method, we conducted 35
independent trials and froze learning every 100 training episodes to evaluate
the learned policies across 200 episodes per method, plotting the average across
episodes and trials. Also shown is one standard deviation in performance.

The results show that COMA is superior to the IAC baselines in all scenarios. Interestingly, the IAC methods also eventually learn reasonable policies in
5m, although they need substantially more episodes to do so. This may seem
counterintuitive since in the IAC methods, the actor and critic networks share
parameters in their early layers (see Section
\ref{ssec:architecture}), which could be expected to speed learning. However, these 
results suggest that the improved accuracy of policy evaluation made 
possible by conditioning on the global state outweighs the overhead of training 
a separate network.

Furthermore, COMA strictly dominates \mbox{central-$QV$}, both in training 
speed and in final performance across all settings. This is a strong indicator 
that our counterfactual baseline is crucial when using a central $Q$-critic to 
train decentralised policies.

Learning a state-value function has the obvious advantage of not conditioning 
on the joint action. Still, we find that COMA outperforms the 
\mbox{central-$V$} baseline in final performance. Furthermore, COMA typically achieves good 
policies faster, which is expected as COMA provides a shaped training signal. 
Training is also more stable  than \mbox{central-$V$}, which is a consequence of the COMA gradient 
tending to zero as the policy becomes greedy. Overall, COMA is the best 
performing and most consistent method.

\citet{usunier2016episodic} report the performance of their best agents trained with their 
state-of-the-art centralised controller labelled GMEZO (greedy-MDP with 
episodic zero-order optimisation), and for a centralised DQN controller, both 
given a full field of view and access to attack-move macro-actions. These 
results are compared in Table~\ref{tbl:test_in_domain} against the best agents 
trained with COMA for each map. Clearly, in most settings these agents achieve 
performance comparable to the best published win rates despite being 
restricted to decentralised policies and local fields of view.

\footnotetext[3]{5w 
	DQN and GMEZO benchmark performances are of a policy 
	trained on a larger 
	map and tested on 5w}

\section{Conclusions \& Future Work}
\label{sec:conclusion}

This paper presented COMA policy gradients, a method that uses a centralised 
critic in order to estimate a counterfactual advantage for decentralised 
policies in mutli-agent RL. COMA addresses the challenges of multi-agent credit 
assignment by using a counterfactual baseline that marginalises out a single 
agent's action, while keeping the other agents' actions fixed. Our results in a 
decentralised \emph{StarCraft unit micromanagement} benchmark show that COMA 
significantly improves final performance and training speed over other 
multi-agent actor-critic methods and remains competitive with state-of-the-art 
centralised controllers under best-performance reporting. Future work will 
extend COMA to tackle scenarios with large numbers of agents, where centralised 
critics are more difficult to train and exploration is harder to coordinate. We 
also aim to develop more  sample-efficient variants that are practical for 
real-world applications such as self-driving cars.

\section*{Acknowledgements} 
This project has received funding from the European Research Council (ERC) under the European Union's Horizon 2020 research and innovation programme (grant agreement number 637713).  It was also supported by the Oxford-Google DeepMind Graduate Scholarship, the 
UK EPSRC CDT in Autonomous Intelligent Machines and Systems, and a generous 
grant from Microsoft for their Azure cloud computing services.
We would like to thank Nando de Freitas, Yannis Assael, and Brendan Shillingford
for helpful comments and discussion. We also thank Gabriel Synnaeve, Zeming Lin,
and the rest of the TorchCraft team at FAIR for their work on the interface.

\section*{Errata}
An earlier version of this paper contained an error in the proof of Lemma \ref{lemma} because the critic depended on the state $s$ but not the joint history $\vtau$.  In this version, Equation \ref{eqn:advantage} and Figure \ref{fig:fig_1} have been updated to add this dependence. In addition, the proof of Lemma \ref{lemma} has been revised to show explicitly how existing policy gradient results apply to this modified setting. The proof also refers to the \citet{sutton1999policy} result instead of that of \citet{konda2000actor} as the latter requires that the Markov chain induced by the policy be irreducible, which does not hold for history-based state representations. Thanks to Frans Oliehoek and Chris Amato for pointing out these issues.  Thanks also to Frans Oliehoek and Andrea Baisero for feedback on the revised proof. See also \citet{lyu2021contrasting,lyu2022deeper,lyu2023centralized,lyu2024centralizedcriticsmultiagentreinforcement} for more details on this topic.

\bibliography{Forester-Farquhar.bbl}
\bibliographystyle{include/aaai}
\onecolumn
	
\appendix

\section{Proof of Lemma \ref{lemma}}
\label{app:proof}

The COMA gradient is given by
\begin{align}
g = \E_{\vpi} \left[ \sum_a \nabla_{\theta} \log
\pi^a(u^a | \tau^a)
A^a(s,\vtau, \vu) \right], \\[2em]
A^a(s,\vtau, \vu)
= Q(s, \vtau, \vu) - b(s, \vtau, \vu^{-a}),
\end{align}
where $\theta$ are the parameters of all actor policies, e.g., $\theta = 
\{\theta^1,\ldots,\theta^{|A|}\}$, and $b(s, \vtau, \vu^{-a})$ is the counterfactual 
baseline defined in equation \ref{eqn:advantage}.

First consider the expected contribution of the baseline:
\begin{align}
g_b &= - \sum_{s, \vtau} p(s, \vtau) \sum_{\vu} \bigg( \vpi(\vu | \vtau) \sum_a \nabla_\theta \log \pi^a(u^a | \tau^a)
b(s, \vtau, \vu^{-a}) \bigg), \\
 &= - \sum_{s, \vtau} p(s, \vtau) \sum_a \sum_{\vu^{-a}} \bigg( 
\vpi^{-a}(\vu^{-a}|\vtau^{-a}) \sum_{u^a} 
\pi^a(u^a|\tau^a)  \nabla_\theta \log \pi^a(u^a | \tau^a)
b(s, \vtau, \vu^{-a}) \bigg), \\
&= - \sum_{s, \vtau} p(s,\vtau) \sum_a \sum_{\vu^{-a}} \bigg( \vpi^{-a}(\vu^{-a}|\vtau^{-a}) 
\sum_{u^a} 
\nabla_\theta \pi^a(u^a | \tau^a)
b(s, \vtau, \vu^{-a}) \bigg), \\
&= - \sum_{s, \vtau} p(s, \vtau) \sum_a \sum_{\vu^{-a}} \bigg( \vpi^{-a}(\vu^{-a}|\vtau^{-a})
b(s, \vtau, \vu^{-a}) \; \nabla_\theta 1 \bigg), \\
&= 0,
\end{align}
where $\vtau^{-a}$ is the joint action-observation history of all agents except $a$, and $\vpi^{-a}$ is the joint policy of all agents except $a$.
Clearly, the per-agent baseline, although it may reduce variance, does not change 
the expected gradient, and therefore does not affect the convergence of COMA. 
The key feature of the counterfactual baseline which allows this property is the independence for agent $a$ on the action $u^a$.

The remainder of the expected policy gradient is given by:
\begin{align}
g &= \E_{\vpi} \left[ \sum_a \nabla_\theta \log
\pi^a(u^a | \tau^a)
Q(s,\vtau, \vu) \right] \\
\label{eq:remainder}
&= \E_{\vpi} \left[ \nabla_\theta \log \prod_a 
\pi^a(u^a | \tau^a)
Q(s,\vtau, \vu) \right].
\end{align}
By writing the joint policy as the product of the independent actors,
\begin{equation}
\label{eq:policy}
\vpi(\vu|\vtau) = \prod_a \pi^a(u^a | \tau^a),
\end{equation}
we can rewrite \eqref{eq:remainder} as:
\begin{align}
\label{eq:joint}
g = \E_{\vpi} \left[ \nabla_\theta \log \vpi(\vu|\vtau) 
Q(s,\vtau, \vu) \right].
\end{align}
Furthermore, we can reinterpret $\vpi$ as a single-agent policy in a corresponding MDP $M=\langle S_m,U_m,P_m,r_m,\gamma\rangle$ where, 
\begin{itemize}
    \item $S_m=S\times \vT$; 
    \item $U_m=\vU$;
    \item $P_m(s'_{m}|s_m) = P(s',\vtau'|s,\vtau) = P(s'|s)\mathds{1}_=(\vz',\vz)$, where $\vtau'$ is $\vtau$ extended with $\vz'$, $\vz=\left[O(s',1),\ldots,O(s',n)\right]$, and $\mathds{1}_=(\vz',\vz) = 1$ iff $\vz'=\vz$; and
    \item $r_m((s,\vtau),\vu) = r(s,\vu)$.
\end{itemize}
The gradient $g$ in the stochastic game $G$ corresponds to the standard policy gradient gradient in $M$:
\begin{align}
\label{eq:KT}
g_m = \E_{\pi_m} \left[ \nabla_\theta \log \pi_m(u_m|s_m) 
Q(s_m,u_m) \right],
\end{align}
where $\pi_m(u_m|s_m) = \pi_m(u_m|s,\vtau) = \vpi(\vu|\vtau)$, i.e., $\pi_m$ follows the policy in \eqref{eq:policy}, foregoing the option of depending on $s$. 
\citet{sutton1999policy} prove that an actor-critic 
following this gradient converges to a local maximum of the expected return, given several assumptions, including:
\begin{enumerate}[leftmargin=3em]
    \item the MDP has bounded rewards;
    \item the policy is differentiable;
    \item the critic is trained against unbiased targets, as in TD(1);
    \item the critic has converged to a local optimum before each policy gradient is estimated; and
	\item the critic uses a representation compatible with the policy.
\end{enumerate}
The policy parameterisation (i.e., the single-agent joint-action learner is decomposed into 
independent actors) is immaterial to convergence, as long as it remains 
differentiable. 

\section{Training Details and Hyperparameters}

Training is performed in batch mode, with a batch size of 30. Due to parameter 
sharing,  all agents can be processed in parallel, with each agent for each 
episode and time step occupying one batch entry. The training cycle progresses 
in three steps (completion of all three steps constitutes as one episode in our 
graphs):
1) \emph{collect data}: collect $\frac{30}{n}$ episodes;
2) \emph{train critic}: for each time step, apply a gradient step to the 
feed-forward critic, starting at the end of the episode; and
3) \emph{train actor}: fully unroll the recurrent part of the actor, aggregate 
gradients in the backward pass across all time steps, and apply a gradient 
update. 

We use a target network for the critic, which updates every $150$ training 
steps for the feed-forward centralised critics and every $50$ steps for the 
recurrent IAC critics. The feed-forward critic receives more learning steps, 
since it performs a parameter update for each timestep. Both the actor and the 
critic networks are trained using RMS-prop with learning rate $0.0005$ and 
alpha $0.99$, without weight decay.  We set gamma to $0.99$ for all maps.

Although tuning the skip-frame in StarCraft can improve absolute 
performance \citep{peng2017multiagent}, we use a default value of 7, since the 
main focus is a relative evaluation between COMA and the baselines.

\section{Algorithm}

\begin{algorithm}
	\caption{Counterfactual Multi-Agent (COMA) Policy Gradients}
	\label{alg:dic}
	\hskip -2em
	\begin{algorithmic} 
		\State Initialise $\theta^c_1$, $\hat{\theta}^c_{1}, \theta^{\pi} $
		\For{each training episode $e$}
		\State Empty buffer
		\For{$e_c=1$ {\bfseries to} $\frac{\text{BatchSize}}{n}$}
		\State $s_1 = $ initial state, $t = 0$, $h^a_0 = \mathbf{0}$ for each 
		agent $a$
		\While{$s_t \ne$ terminal {\bfseries and} $t < T$}
		\State $t = t + 1$
		\For{each agent $a$}
		\State $h^a_t = \text{Actor}\left(o^a_t,h^a_{t-1}, u^a_{t-1},a,u; 
		\theta_{i} \right)$ 
		\State Sample $u^a_t$ from $\pi(h^a_{t}, \epsilon(e)) $ 
		\EndFor
		\State Get reward $r_t$ and next state $s_{t+1}$    \EndWhile
		\State  Add episode to buffer
		\EndFor
		\State Collate episodes in buffer into single batch
		\For{$t=1$ {\bfseries to} $T$}  // from now processing all agents in 
		parallel via single batch
		\State Batch unroll RNN using states, actions and rewards
		\State Calculate TD($\lambda$) targets $y^a_t$ using  
		$\hat{\theta}^c_{i}$
		\EndFor
		\For{$t=T$ {\bfseries down to} $1$} 
		\State $ \Delta Q^a_t  = y^a_t - Q\left( s^a_j, \vu \right) $
		\State  $\Delta \theta^c = \nabla_{
			\theta^c} (\Delta Q^a_t)^2 $ // calculate critic gradient
		\State  $\theta^{c}_{i+1} = \theta^{c}_i - \alpha \Delta \theta^{c}$ // 
		update critic weights
		\State Every C steps  reset $\hat{\theta}^c_{i} = \theta^c_{i}$
		
		\EndFor
		\For{$t=T$ {\bfseries down to} $1$}
		\State $A^a( s^a_t,\vu) =Q( s^a_t, \vu) - \sum_u  Q( s^a_t, u, 
		\vu^{-a}) \pi({u|h^a_t} )$ // calculate COMA
		\State  $\Delta \theta^{\pi} =  \Delta \theta^{\pi}  + \nabla_{ 
		\theta^{\pi}} \log  \pi({u|h^a_t} ) A^a( s^a_t,\vu) $ // 
		accumulate actor gradients
		\EndFor
		\State	$\theta^{\pi}_{i+1} = \theta^{\pi}_i + \alpha \Delta 
		\theta^{\pi}$ // update actor weights
		\EndFor
	\end{algorithmic}
\end{algorithm}

\end{document}